\documentclass[letterpaper, 10 pt, conference]{ieeeconf}
\IEEEoverridecommandlockouts
\usepackage{cite}
\usepackage{algorithmic}
\usepackage{graphicx}
\usepackage{textcomp}
\usepackage{xcolor}
\usepackage{comment}
\usepackage{amssymb}
\usepackage{amsmath}
\usepackage{siunitx}
\usepackage{soul}
\usepackage{subcaption}
\usepackage{siunitx}
\usepackage[font=footnotesize]{caption}
\usepackage{hyperref}

\newtheorem{thm}{Theorem}[section]

\newtheorem{prop}[thm]{Proposition}
\newtheorem{cor}{Corollary}

\newtheorem{defn}{Definition}[section]

\newcommand*{\Ye}[1]{$\spadesuit$\footnote{\color{blue}{Ye}: #1}}

\newcommand*{\Jonas}[1]{$\diamondsuit$\footnote{\color{purple}{Jonas}: #1}}

\def\BibTeX{{\rm B\kern-.05em{\sc i\kern-.025em b}\kern-.08em
    T\kern-.1667em\lower.7ex\hbox{E}\kern-.125emX}}
    
\begin{document}

\title{Paper Title*\\
{\footnotesize \textsuperscript{*}Note: Sub-titles are not captured in Xplore and
should not be used}
\thanks{Identify applicable funding agency here. If none, delete this.}
}

\title{\LARGE \bf
Towards Safe Locomotion Navigation in Partially Observable Environments with Uneven Terrain
}

\author{Jonas Warnke$^{\star}$, Abdulaziz Shamsah$^{\star}$, Yingke Li$^{\star}$, and Ye Zhao$^{*}$
\thanks{$^{*}$The authors are with the Laboratory for Intelligent Decision and Autonomous Robots, Woodruff School of Mechanical Engineering, Georgia Institute of Technology, Atlanta, GA 30313, USA
        {\tt\small \{jwarnke, ashamsah3, yli3225, yzhao301\}@gatech.edu}}%
\thanks{$^{\star}$the first three authors are equally contributed.}
}

\maketitle

\begin{abstract}
This study proposes an integrated task and motion planning method for dynamic locomotion in partially observable environments with multi-level safety guarantees. This layered planning framework is composed of a high-level symbolic task planner and a low-level phase-space motion planner. A belief abstraction at the task planning level enables belief estimation of dynamic obstacle locations and guarantees navigation safety with collision avoidance. The high-level task planner, i.e., a two-level navigation planner, employs linear temporal logic for a reactive game synthesis between the robot and its environment while incorporating low-level safe keyframe policies into formal task specification design. The synthesized task planner commands a series of locomotion actions including walking step length, step height, and heading angle changes, to the underlying keyframe decision-maker, which further determines the robot center-of-mass apex velocity keyframe. The low-level phase-space planner uses a reduced-order locomotion model to generate non-periodic trajectories meeting balancing safety criteria for straight and steering walking. These criteria are characterized by constraints on locomotion keyframe states, and are used to define keyframe transition policies via viability kernels. Simulation results of a Cassie bipedal robot designed by Agility Robotics demonstrate locomotion maneuvering in a three-dimensional, partially observable environment consisting of dynamic obstacles and uneven terrain.
\end{abstract}

%
\section{Introduction}
Safety scalable to high-dimensional robotic systems becomes imperative as legged robots maneuver over uneven and unpredictable environments, and ought to be reasoned about from various perspectives such as balance and collision avoidance.
In the robot mobility field, navigation safety is conventionally studied from the collision avoidance perspective \cite{fox1997dynamic, kousik2017safe, bajcsy2019scalable}. However, in the context of dynamic legged locomotion, maintaining dynamic balancing, i.e., avoiding a fall \cite{koolen2012capturability, heim2019beyond, luo2019robust}, 
becomes an essential safety criterion. Reasoning about safety from both levels has been largely under-explored in the field. As one closely-related line of research, Wieber's recent studies \cite{bohorquez2016safe, pajon2019safe} proposed a model predictive control (MPC) method to address safe navigation problems for bipedal walking robots in a crowded environment. Nevertheless, their work mainly focused on \textit{passive} safety, i.e., the robot comes to a stop for collision avoidance. Their MPC optimization weighs the safety criteria and lacks formal guarantees on navigation safety. To address these issues, our method takes one step towards using a symbolic planning method to design \textit{active} navigation decisions for safety guarantees, i.e., the robot steers its walking direction to avoid collisions besides the coming-to-a-stop strategy. Meanwhile, we incorporate a belief abstraction approach to assure safe navigation in a partially observable environment. 

\begin{figure}[t]
\centerline{\includegraphics[width=.46\textwidth]{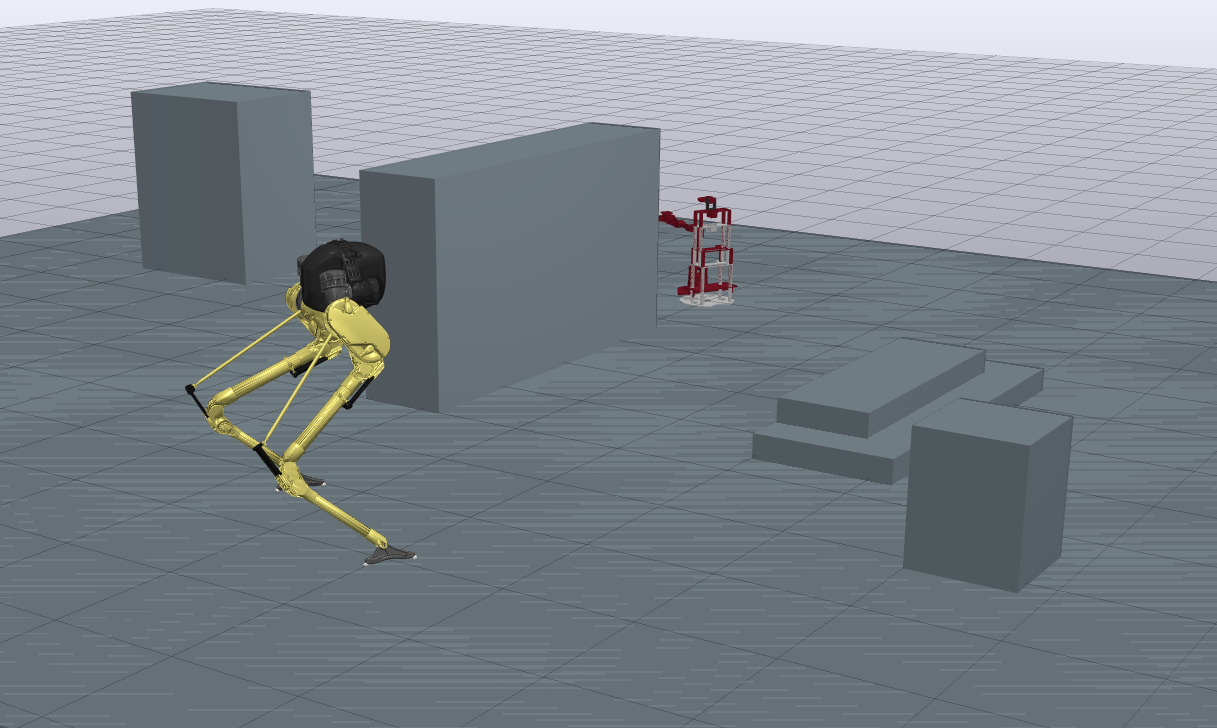}}
\caption{
An illustration of safe locomotion navigation in a crowded environment simulation with dynamic mobile robot obstacles and uneven terrain.}
\label{fig:sim}
\vspace{-0.15in}
\end{figure}



Phase-space planning (PSP) is a general planning framework for non-periodic dynamic legged locomotion over rough terrain \cite{zhao2012three, zhao2017robust}. This study further generalizes the previous PSP framework by (i) studying more complex locomotion scenarios by incorporating navigation keyframe states; (ii) avoiding collisions with dynamic obstacles in partially observable 3D environments (see Fig.~\ref{fig:sim}) (iii) taking safety criteria into account for viability kernel and keyframe policy design. Compared to other well-established locomotion planning frameworks -- Capture Point \cite{koolen2012capturability}, Divergent Component of Motion \cite{englsberger2015three}, Zero Moment Point \cite{vukobratovic2004zero} -- our framework has a large focus on providing safety guarantees for simultaneously maintaining balance and avoiding collisions in partially observable environments with rough terrain.

Temporal-logic-based motion planning has been widely studied for mobile robot navigation in partially observable domains through exploration \cite{Exploring_Partially_Known}, 
re-synthesis \cite{Temporal_hybrid_systems} when encountering unexpected obstacles, and receding-horizon planning \cite{RecHorCont}. These approaches are better suited for guaranteeing successful navigation and collision avoidance in environments with static obstacles and simple robot dynamics.
On the contrary, our framework takes into account dynamic obstacles that are only visible within a limited range. We devise a variant of the approach in \cite{bharadwaj2018synthesis} via a combined \textit{top-down} and \textit{bottom-up} strategy to design the navigation strategy in a partially observable environment while guaranteeing collision avoidance. This work generalizes our previous work on temporal-logic-based locomotion \cite{zhao2016high, zhao2018reactive, Kulgod2020LTL}.

A challenge of linear-temporal-logic-based navigation planning is to guarantee successful execution of the commands from the high-level planner in the presence of complex low-level robot dynamics. Our study explicitly addresses this challenge by encoding low-level physics-consistent safety criteria into the high-level task specification design. This strategy ensures that, on top of collision avoidance, the task planner commands actions that can be safely executed by the low-level planner. The safety properties are expressed as {\textit{viability kernel}} via viability theory \cite{aubin2009viability, wieber2008viability}. This safety-coherent hierarchy is scalable to more complex robot systems and environments through appropriate specifications.



The main contributions of this study are fourfold:
(i) design two-level safety criteria for locomotion motion planning, which guarantees the simultaneous dynamic balancing and navigation safety as well as waypoint tracking. (ii) devise a keyframe transition map based on low-level motion planer constraints and design a keyframe policy for locomotion safe navigation.
(iii) employ a belief abstraction method for a reactive navigation game
to expand navigation choices in a partially observable environment.
(iv) design a hierarchical planning structure that integrates safety for the high-level task planner and low-level motion planner cohesively.



    







\begin{figure}[t]
\centerline{\includegraphics[width=.45\textwidth]{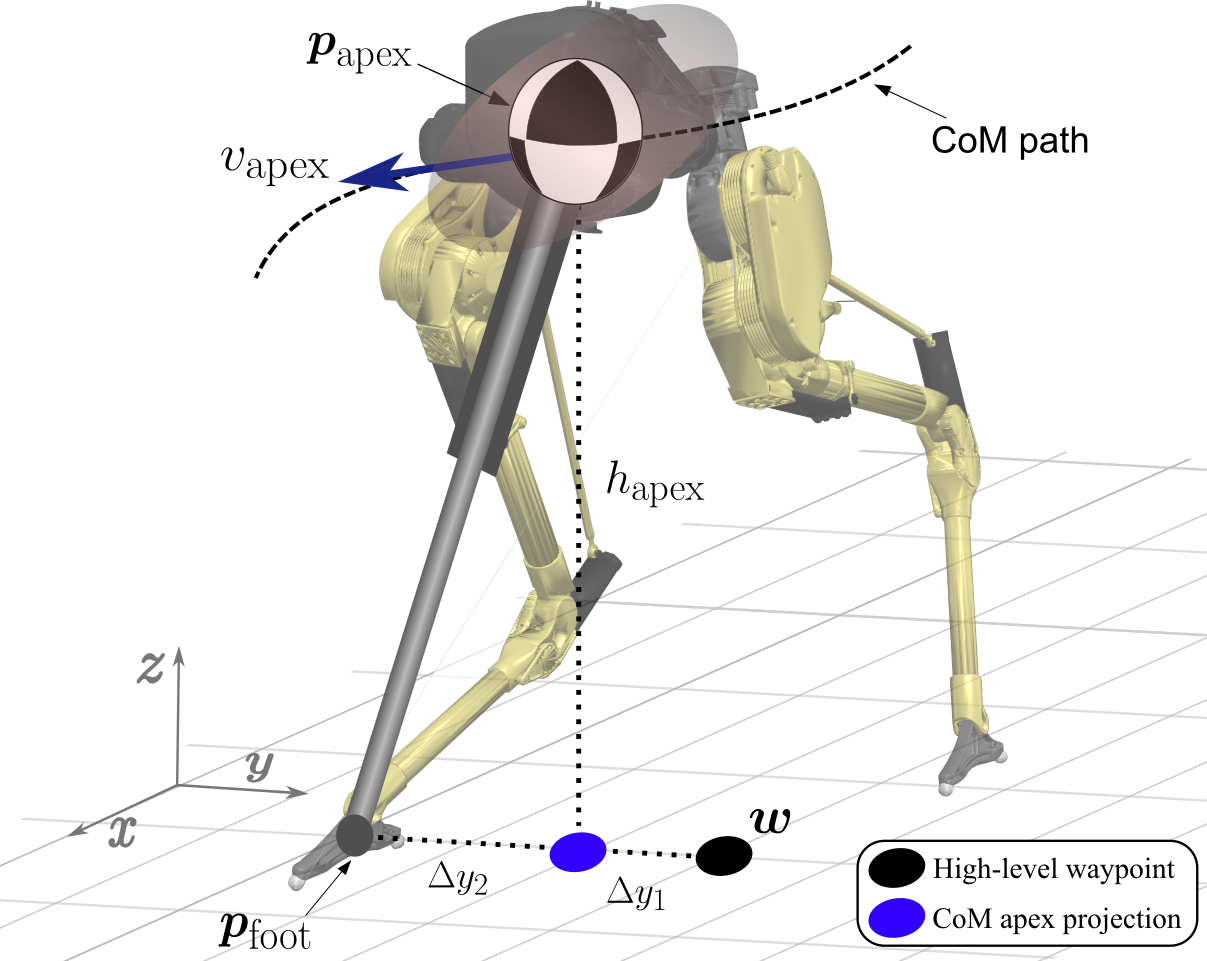}}
\caption{Reduced-order modeling of Cassie robot as a 3D prismatic inverted pendulum model with all of its mass concentrated on its CoM and a telescopic leg to comply to the varying CoM height. $\Delta y_1$ is the relative lateral distance between lateral CoM apex position and the high-level waypoint $\boldsymbol{w}$, and $\Delta y_2$ is the lateral distance between the CoM lateral apex position and the lateral foot placement.}
\label{fig:notation}
\vspace{-0.15in}
\end{figure}

\section{Safe Locomotion Planning}
\label{sec:motion planner}
This section will introduce a locomotion planner based on a reduced-order model and then propose safety locomotion criteria for different walking scenarios. The reduced-order model refers to the dynamics of the prismatic inverted pendulum model \cite{zhao2017robust} in our study, and is used to derive an analytical solution for the robot phase-space trajectories.

\subsection{Reduced-order Locomotion Planning}

This subsection first introduces mathematical notations of our reduced-order model. As shown in Fig. \ref{fig:notation}, 
the center-of-mass (CoM) position $\boldsymbol{p} = (x, y, z)^T$ is composed of the sagittal, lateral, and vertical positions. We denote the apex position as $\boldsymbol{p}_{{\rm apex}}=(x_{{\rm apex}},y_{{\rm apex}}, z_{{\rm apex}})^T$, the foot placement as $\boldsymbol{p}_{{\rm foot}}=(x_{{\rm foot}},y_{{\rm foot}}, z_{{\rm foot}})^T$, and $h_{\rm apex}$ is the relative apex CoM height with respect to the stance foot height. $v_{{\rm apex}}$ denotes the CoM velocity at $\boldsymbol{p}_{{\rm apex}}$. $\Delta y_1$ is the relative lateral distance between lateral apex position and the high-level waypoint $\boldsymbol{w}$. $\Delta y_2 :=y_{{\rm apex}}-y_{{\rm foot}}$ is defined to be the lateral distance between the CoM lateral apex position and the lateral foot placement. This parameter will be used to determine the allowable steering angle in Section~\ref{subsec:sc}.

\begin{figure}[t]
\centerline{\includegraphics[width=.5\textwidth]{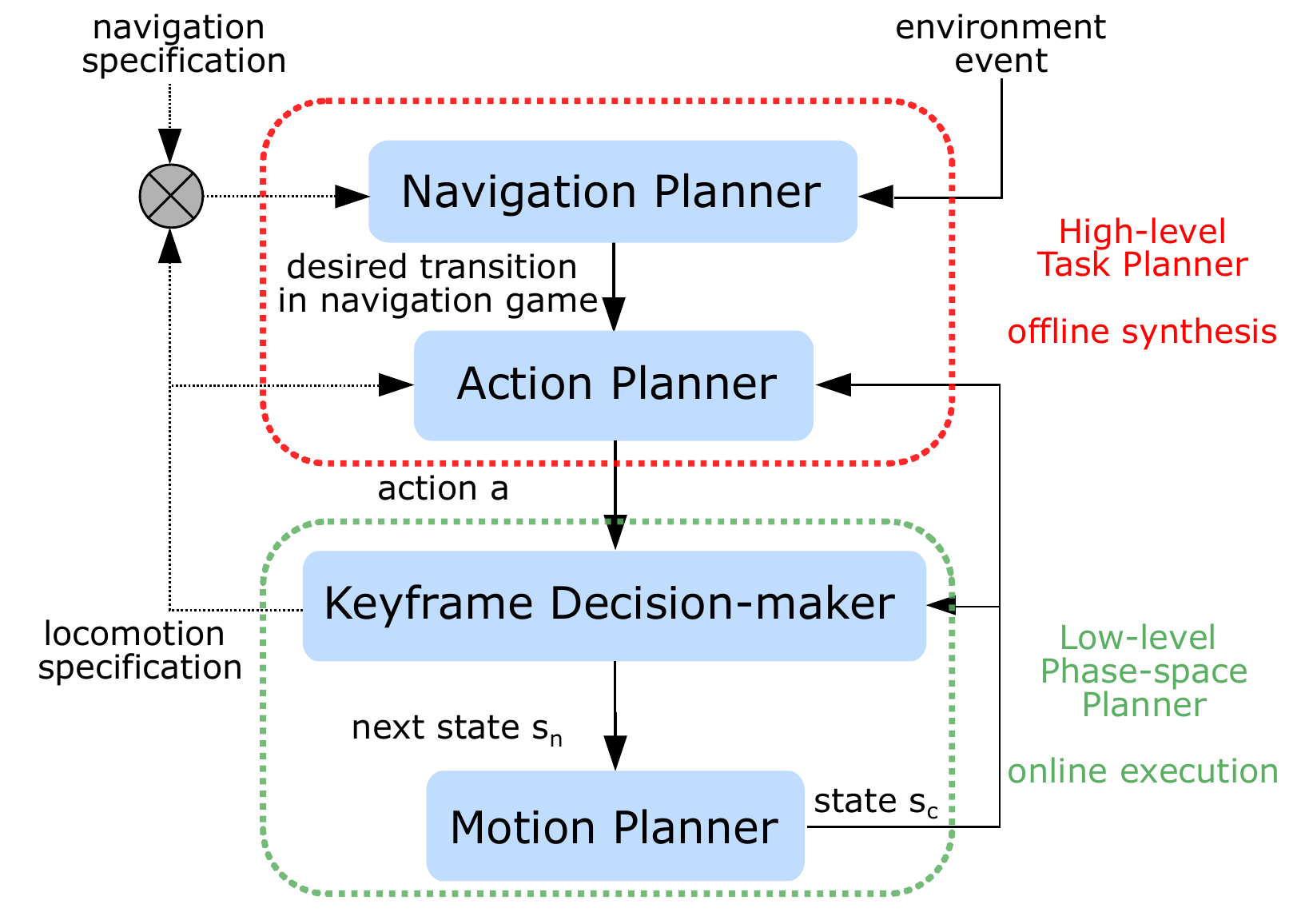}}
\caption{Block diagram of the proposed locomotion planning framework. The task planner employs a linear temporal logic approach to synthesize actions. At the low-level, the keyframe decision-maker generates the keyframe states sent to the motion planner. Locomotion specifications from the low-level will be incorporated into the task planner. Details of the state and action is introduced in Definition~\ref{def:keyframe}. More discussions will be in Section~\ref{sec:results}.}
\label{fig:frame}
\vspace{-0.15in}
\end{figure}

Phase space planning is a keyframe-based non-periodic planning method for dynamic locomotion \cite{zhao2017robust}. 
Our study generalizes the keyframe definition in our previous work by introducing diverse navigation actions in 3D environments. 

\begin{figure*}[th]
\centerline{\includegraphics[width=0.97\textwidth]{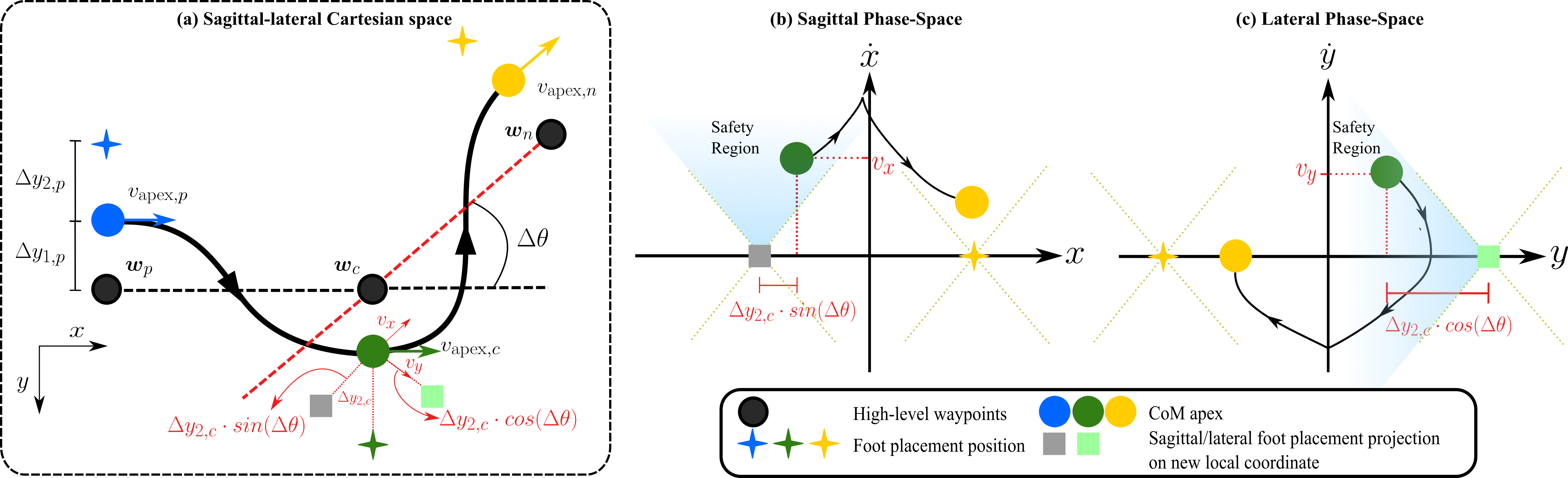}}
\caption{Phase-space safety region for steering walking: (a) shows three consecutive keyframes with a heading angle change ($\Delta \theta$) between the current keyframe and the next keyframe. The CoM sagittal-lateral geometric trajectory is represented by the solid thick black line. The direction change introduces a new local coordinate shown in red dashed line. Subfigures (b) and (c) show the sagittal and lateral phase-space plots respectively, both satisfying the safety criteria proposed in Proposition \ref{prop:steering1}. The subscripts $p$, $c$ and $n$ denote the previous, current, and next walking steps, respectively.}
\label{fig:steering_safety}
\vspace{-0.15in}
\end{figure*}

\begin{defn}[Locomotion Keyframe State]\label{def:keyframe}
A keyframe state of our reduced-order model is defined as $\boldsymbol{k} = (d, \Delta\theta, \Delta z_{\rm foot}$, $i_{\rm st}, c_{\rm stop}, v_{\rm apex}, z_{\rm apex}$) $ \in \mathcal{K}$, where 
\begin{itemize}
    \item $d := x_{{\rm apex},n}-x_{{\rm apex},c}$ is the walking step length~\footnote{while in straight walking $d$ represents the step length, the step length during steering walking is adjusted to reach the next waypoint on the new local coordinate.};
    \item $\Delta \theta:=\theta_{v_{{\rm apex},n}}-\theta_{v_{{\rm apex},c}}$ is the heading angle change at two consecutive CoM apex states;
    \item $\Delta z_{\rm foot} := z_{{\rm foot},n}-z_{{\rm foot},c}$ is the height change for successive foot placements;
    \item $i_{\rm st}$ is the desired stance foot index;
    \item $c_{\rm stop}$ is a boolean informing the motion planner to stop (\textsf{stop}) at the next keyframe;
    \item $v_{\rm apex}$ is the CoM sagittal apex velocity;
    \item $z_{\rm apex}$ is the apex CoM height with respect to the absolute zero height reference, selected as the level ground height in this study.
\end{itemize}
\end{defn}
The parameters $d$, $\Delta\theta$, $\Delta z_{\rm foot}$, $i_{\rm st}$, and $c_{\rm stop}$ are determined by the navigation policy that will be designed in the task planning section. These six parameters are defined as the action, i.e., $\boldsymbol{a} = (d, \Delta\theta, \Delta z_{\rm foot}, i_{\rm st}, c_{\rm stop}) \in \mathcal{A}$. We represent the parameters $d$, $\Delta\theta$, and $\Delta z_{\rm foot}$ in the cartesian space with the high-level waypoints $\boldsymbol{w}$.
On the other hand, the state is $\boldsymbol{s} = (v_{\rm apex}, z_{\rm apex}) \in \mathcal{S}$. 
The state and action will be used to define a keyframe transition map in Section~\ref{sec:interface}.

When the CoM motion is constrained within a piece-wise linear surface parameterized by $h = k(x -x_{{\rm foot}})+h_{\rm apex}$, where $h$ denotes the CoM height from the stance foot, the reduced-order model becomes linear and an analytical solution exists: 
\begin{equation}\label{eqn:analyticalsolution}
    \dot{p}_{\rm com} = \pm \sqrt{\omega^{2}((p_{\rm com} -p_{\rm foot})^{2}-(p_{0}-p_{\rm foot})^{2})+\dot{p}_{0}^{2}}
\end{equation}
where the asymptote slope $\omega = \sqrt{g/h_{\rm apex}}$. Note that Eq. (\ref{eqn:analyticalsolution}) holds for both sagittal and lateral directions, i.e., $p_{\rm com} \in \{x, y\}$. The initial condition $(p_0, \dot{p}_0)$ is chosen as the CoM apex state. Detailed derivations are elaborated in the Appendix.

\subsection{Safe Locomotion Criteria}
\label{subsec:sc}
Avoiding a fall is an essential capability of dynamic legged locomotion. Numerous studies have been proposed to quantify locomotion safety and design recovery controllers \cite{heim2019beyond, stephens2007humanoid}. Before proposing safety locomotion criteria, let us first define locomotion balancing safety.

\begin{defn}[Balancing Safety]
The balancing safety region $\mathcal{R}_s$ for one locomotion step is defined as the set of viable keyframe states $\boldsymbol{k} \in \mathcal{K}$ such that the robot maintains its balance, i.e., avoids a fall.
\label{def:balance_saf}
\end{defn}
Note that, the keyframe state $\boldsymbol{k}$ includes the action $\boldsymbol{a}$ so the control is implicit in the balancing safety region. 

%
The balancing safety region $\mathcal{R}_s$ requires satisfying multiple safety criteria that will be proposed in Propositions~\ref{prop:steering1}-\ref{prop:straight}. We will delve into safety criteria for both straight and steering walking.
As a general principle of balancing safety, the sagittal CoM position should be able to cross the sagittal apex with a positive CoM velocity while the lateral CoM velocity should be able to reach the zero lateral velocity threshold at the next apex state. Ruling out the fall situations provides us upper and lower bounds of the balancing safety region. 
The safety criteria are proposed as follows.
\begin{prop}\label{prop:steering1}
For steering walking, 
the current sagittal CoM apex velocity $v_{{\rm apex},c}$ in the original local coordinate is bounded by
\begin{equation}
    \Delta y_{2,c} \cdot \omega \cdot \tan{\Delta\theta} \leq v_{{\rm apex},c} \leq \frac{\Delta y_{2,c} \cdot \omega}{\tan{\Delta\theta}}
    \vspace{0.05in}
\end{equation}
\end{prop}

A fall will occur when $v_{{\rm apex},c}$ is out of this safety range such that either the lateral CoM velocity cannot reach zero at the next apex state or the sagittal CoM can not climb over the next sagittal CoM apex. Fig.~\ref{fig:steering_safety} shows a steering walking trajectory and phase-space plot that satisfy Proposition~\ref{prop:steering1}. Namely, the CoM in the sagittal and lateral phase-space should not cross the asymptote line of the shaded safety region as seen in Fig.~\ref{fig:steering_safety}. This criterion is specific to steering walking, as the heading change ($\Delta \theta$) introduces a new local frame, which yields the current state $\boldsymbol{s}_c$ to no longer be an apex state in the new coordinate. As such, it has non-apex sagittal and lateral components, i.e., $v_{y,c} \neq 0$, and $x_{\rm{apex},c} \neq x_{\rm{foot},c}$.
%
Next, we study the constraints between apex velocities of two consecutive walking steps and propose the following proposition and corollaries. 
%


\begin{prop}\label{prop:straight}
For straight walking, given $d$ and $\omega$, the apex velocity for two consecutive walking steps ought to satisfy the following velocity constraint:
\begin{equation}\label{eq:straight}
    -\omega^{2}d^{2} \le v_{{\rm apex},n}^{2}-v_{{\rm apex},c}^{2} \le \omega^{2}d^{2}
\end{equation}
where $d^2 = (x_{{\rm apex},n}-x_{{\rm apex},c})(x_{{\rm apex},c}+x_{{\rm apex},n}-2x_{{\rm foot},c})$.
\end{prop}
\begin{proof}
First, the sagittal switching position can be obtained from the analytical solution in Eq.~(\ref{eqn:analyticalsolution}):
\begin{equation}
    x_{\rm switch} = \frac{1}{2}(\frac{C}{x_{{\rm foot}, n}-x_{{\rm foot}, c}}+(x_{{\rm foot}, c}+x_{{\rm foot}, n}))
\end{equation}
where $C = (x_{{\rm apex},c}-x_{{\rm foot},c})^{2}-(x_{{\rm apex},n}-x_{{\rm foot}, n})^{2} +(\dot{x}_{{\rm apex},n}^{2}-\dot{x}_{{\rm apex},c}^{2})/\omega^{2}$. This walking step switching position is required to stay between the two consecutive CoM apex positions, i.e., 
\begin{equation}
    x_{{\rm apex},c} \le x_{\rm switch} \le x_{{\rm apex},n}
\end{equation}
which introduces the sagittal apex velocity constraints for two consecutive keyframes as follows.
\begin{equation}
\begin{split}
    \omega^{2}(x_{{\rm apex},n}-x_{{\rm apex},c})&(x_{{\rm apex},c}+x_{{\rm apex},n}-2x_{{\rm foot},n}) \\
    \le \; \dot{x}_{{\rm apex},n}^{2} & -\dot{x}_{{\rm apex},c}^{2} \le  \\
    \omega^{2}(x_{{\rm apex},n}-x_{{\rm apex},c})&(x_{{\rm apex},c}+x_{{\rm apex},n}-2x_{{\rm foot},c})
\end{split}
\end{equation}

Given this bounded difference between two consecutive CoM apex velocity squares, the corresponding safe criterion for straight walking can be expressed as Eq.~(\ref{eq:straight}).
\end{proof}

\begin{cor}
For steering walking in Proposition~\ref{prop:steering1}, given $d$, $\Delta \theta$, $\Delta y_{2,c}$ and $\omega$, two consecutive apex velocities ought to satisfy the following velocity constraint:
\begin{equation}\label{eq:steering1}
    -\omega^{2}d^{2} \le v_{{\rm apex},n}^{2}-(v_{{\rm apex},c}\cos{\Delta\theta})^{2} \le \omega^{2}d^2_{+}
\end{equation}
where $d^2_{+} = d^{2}+2\Delta y_{2,c}d\sin{\Delta\theta}$.
\end{cor}

\begin{cor}
For steering walking in Proposition~\ref{prop:steering1}, similarly, given $d$, $\Delta \theta$, $\Delta y_{2,c}$, and $\omega$, two consecutive apex velocities ought to satisfy the following velocity constraints,
\begin{equation}\label{eq:steering2}
    -\omega^{2}d^{2} \le v_{{\rm apex},n}^{2}-(v_{{\rm apex},c}\cos{\Delta\theta})^{2} \le \omega^{2}d^2_{-}
    \vspace{0.05in}
\end{equation}
where $d^2_{-} = d^{2} - 2\Delta y_{2,c}d\sin{\Delta\theta}$. Note that, parameters $v_{{\rm apex},n}$, $d$, and $\Delta \theta$ in Eqs.~(\ref{eq:straight})-(\ref{eq:steering2}) are the keyframe states.
\end{cor}

\section{Keyframe decision-making}
\label{sec:interface}
Given the aforementioned keyframe-based safety criteria for one walking step, we now focus on the keyframe decision-making, an interface between the high-level and low-level planners as seen in Fig.~\ref{fig:frame}. Given a high-level action, the keyframe decision-maker will choose appropriate keyframe states for the motion planner. 
According to the keyframe definition in Def.~\ref{def:keyframe}, we propose the following non-deterministic keyframe transition map. 
\begin{defn}[Transition Map]
A keyframe transition map is non-deterministic and defined as $\boldsymbol{s}_n = T(\boldsymbol{s}_c, \boldsymbol{a})$ where the action $\boldsymbol{a} = (d, \Delta\theta, \Delta z_{\rm foot}, i_{\rm st}, c_{\rm stop})$, the state $\boldsymbol{s}_i=(v_{{\rm apex}, i}, z_{{\rm apex}, i}) \in \mathcal{S}_i$, $i \in \{c, n\}$ denotes current and next walking step indices, respectively.
\end{defn}
The objective of our keyframe decision-maker is: given an action $\boldsymbol{a}$ from the task planner and the current state $\boldsymbol{s}_c$, a transition policy will make a decision on the next state $\boldsymbol{s}_n$.

To define this keyframe transition map, we will first investigate the viability of a keyframe transition map and then use it to design a policy of choosing safe keyframe states and the induced task specifications.

\begin{figure}[t]
\centerline{\includegraphics[width=0.69\columnwidth]{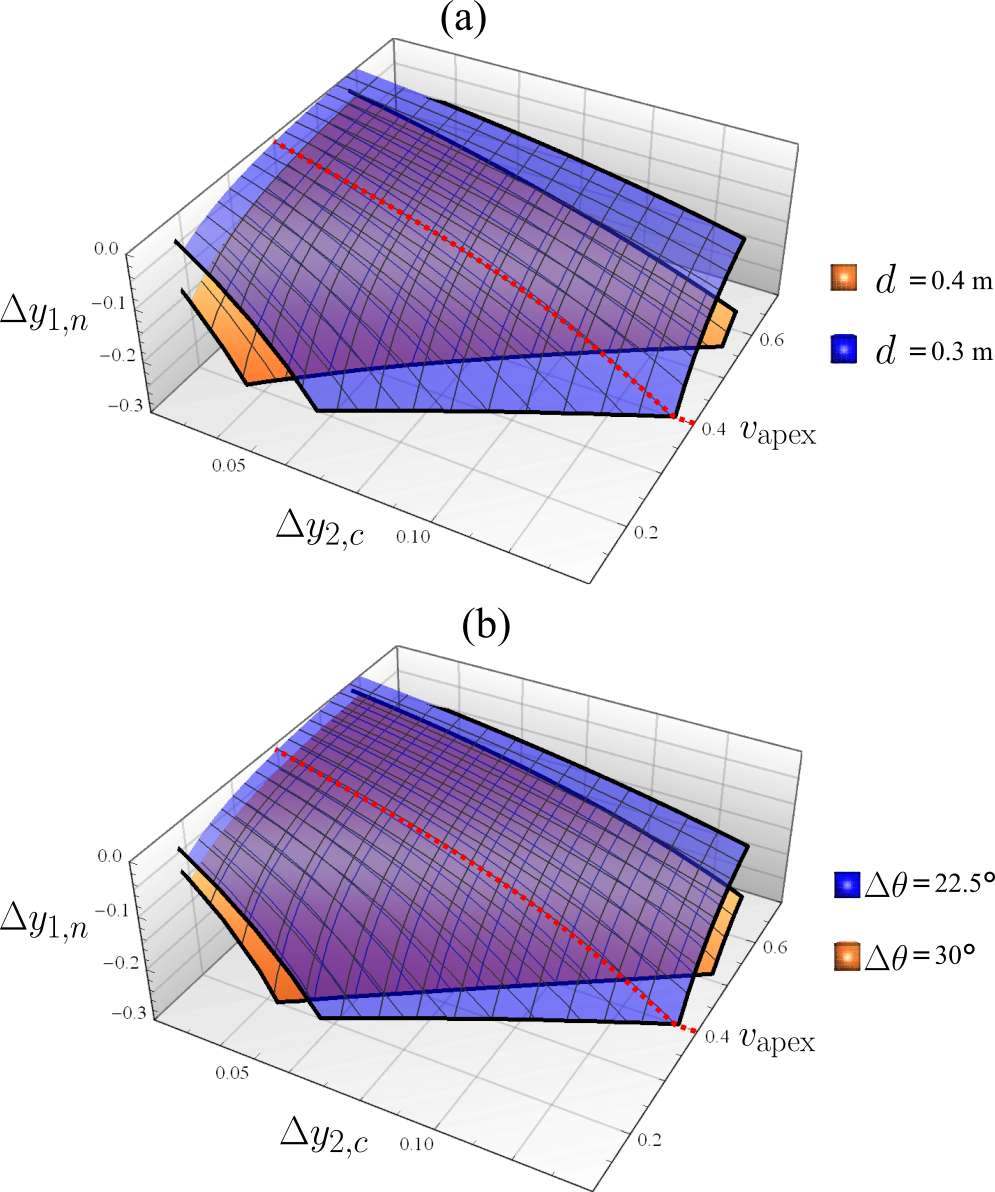}}
\caption{An illustration of keyframe policy design from the viability mapping. Both subfigures (a) and (b) corresponds to steering walking to the right with the right foot in stance. (a) shows $\Delta y_{1,n}$ for steering walking with $\Delta \theta = 22.5^{\circ}$ for two different step lengths $0.3$ and $0.4$ m, as a function of both $v_{\rm {apex}}$ and $\Delta y_{2,c}$. In (b), it shows $\Delta y_{1,n}$ for steering walking with $d$ = 0.3 m for two different heading angle changes ($\Delta \theta$), as a function of both $v_{\rm {apex}}$ and $\Delta y_{2,c}$. In this mapping, $v_{\rm {apex},c}$ and $v_{\rm {apex},n}$ are equal. In both (a) and (b), the blue surface represents a more robust keyframe transition policy, since it allows for a wider range of $v_{\rm apex}$ and $\Delta y_{2,c}$ that yields $\Delta y_{1,n} \in \mathcal{R}_{\Delta y_1}$.  The red dotted line represents an example of viable $\Delta y_{1,n}$ range given $v_{\rm apex} = 0.4$ m/s in the sampled range.}
\label{fig:vk}
\vspace{-0.15in}
\end{figure}

\subsection{Keyframe Transition Map Viability}
\label{subsec:viable transition map}
%


To achieve locomotion transition viability, the CoM trajectory needs to (i) maintain balancing safety in Def.~\ref{def:balance_saf}, and (ii) accurately track the high-level waypoints $\boldsymbol{w}$. As illustrated in Sec.~\ref{subsec:sc}, to maintain balancing safety, the constraint for $\Delta y_2$ in Proposition~\ref{prop:steering1} should be satisfied at each walking step. To this end, we choose $\Delta y_2$ as a safety indicator of the apex state $\boldsymbol{s}$ and analyze how $\Delta y_2$ varies with respect to different keyframe transition maps. Given the safety indicator $\Delta y_{2,c}$ for the current step, the safety indicator $\Delta y_{2,n}$ for the next step is determined uniquely by ($\boldsymbol{s}_c, \boldsymbol{s}_n, \boldsymbol{a}$) and the locomotion dynamics in Eq.~(\ref{eqn:analyticalsolution}). As shown in Proposition \ref{prop:steering1}, $v_{{\rm apex}, c}$ is directly bounded by $\Delta y_{2,c}$. Thus, the viability of the safety indicator represents a precondition of this viability for the keyframe transition map $\boldsymbol{s}_n = T(\boldsymbol{s}_c, \boldsymbol{a})$. To quantify whether the CoM trajectory tracks the high-level waypoints $\boldsymbol{w}$, we use $\Delta y_1$ as a tracking measure. Namely, $\Delta y_1$ needs to be within a bounded range, an additional viability criterion of the transition map.
\begin{thm}
The keyframe transition map $\boldsymbol{s}_n = T(\boldsymbol{s}_c, \boldsymbol{a})$ is viable only if (i) the transition satisfies the safe criteria in Propositions~\ref{prop:steering1}-\ref{prop:straight}, (ii) the safety indicator $\Delta y_{2}$ and tracking indicator $\Delta y_{1}$ are bounded within their respective viable ranges, i.e., $\Delta y_{2} \in \mathcal{R}_{\Delta y_2}$ and $\Delta y_{1} \in \mathcal{R}_{\Delta y_1}$, and (iii) $\Delta y_1$ and $\Delta y_2$ have a matching sign which alternates between consecutive keyframes across two walking steps.
\label{thm:viable_trans}
\end{thm}





To obtain the viable keyframe transition map, we sample different keyframe transitions that satisfy Theorem~\ref{thm:viable_trans}. An example of this analysis is shown in Fig. \ref{fig:vk}. For each given action, we sample different combinations of $\Delta y_{2,c}$ and  $v_{{\rm apex}}$, and determine $\Delta y_{1,n}$ consequently. Each point in the 3D figure represents a unique, viable keyframe transition.
We heuristically choose the viable range for $\Delta y_1$ to be $[-0.3, 0.3]$ m to limit the amplitude of the lateral CoM oscillations. In Fig.~\ref{fig:vk}(a), it is observed that steering walking with a shorter step length would allow a wider range of $v_{\rm apex}$ and $\Delta y_{2,c}$. The same conclusion is reached when choosing the heading change angle ($\Delta \theta$) as shown in Fig.~\ref{fig:vk}(b). The analysis shown in Fig.~\ref{fig:vk} is an example of this mapping, and other mappings with different keyframe parameter set-ups are implemented in a similar way but not shown due to space limit. For example, a similar mapping is designed for $\Delta y_2$, to maintain $\Delta y_{2,n}$ to be within the viable range $[-0.2, 0.2]$ m given the Cassie leg configuration.
\subsection{Viability-Kernel-Guided Keyframe Policy}
\label{subsec:keyframe-policy}
As illustrated in the viable keyframe transition map in Sec.\ref{subsec:viable transition map},
given an action $\boldsymbol{a}$ and the current state $\boldsymbol{s}_{c}$, there are multiple choices for the next state $\boldsymbol{s}_{n}$ since the keyframe transition map is non-deterministic. In this subsection, we will design a keyframe policy to choose a deterministic next state $\boldsymbol{s}_{n}$ based on $\boldsymbol{s}_{c}$ and $\boldsymbol{a}$.

\begin{defn}[Locomotion Keyframe Policy]
The keyframe policy is a deterministic keyframe transition map $s_n = \mathcal{P}(s_c, a)$ that satisfies the viable keyframe transition map $s_n = \mathcal{T}(s_c, a)$ while obeying the following set of locomotion properties under different walking scenarios.
\begin{itemize}
   \item For straight, obstacle-free walking on level ground, (i) the apex velocity is continuous within the allowable  \textsf{small}, \textsf{medium}, and \textsf{large} value ranges.\footnote{In this study, we choose $[0.1, 0.3]$ m/s, $[0.3, 0.4]$ m/s, and $[0.4, 0.45]$ m/s as the \textsf{small}, \textsf{medium}, and \textsf{large} value ranges for $v_{\rm apex}$. The granularity in $v_{\rm apex}$ between two consecutive keyframes is $0.05$ m/s.} (ii) The step length ($d$) is fixed to $0.4$ m.\footnote{The step length value during straight walking needs to be a multiple of $0.1$ m to adhere to high-level constraint (See Sec.~\ref{subsec:action_planner}).} Given those specifications, $\Delta y_1$ and $\Delta y_2$ are guaranteed to be within their respective viable ranges.
   \item For straight walking on flat ground with an obstacle appearing in the front, the robot will either come to a stop or switch to the steering walking introduced next.
   \item For steering walking, to guarantee that $\Delta y_{1,n}$ and $\Delta y_{2,n}$ are within their viable ranges and $v_{\rm apex}$ is within the safety region (Proposition \ref{prop:steering1}), our keyframe policy will require (i) a \textsf{small} $v_{\rm apex}$ value, (ii) $\Delta \theta = \pm 22.5^\circ$, (iii) a \textsf{large} step length $d$ when steering in the direction opposite to the foot stance, and (iv) a \textsf{small} $d$ when steering in the direction matching the foot stance.\footnote{$[0.2, 0.3]$ m, $[0.3, 0.4]$ m, and $[0.4, 0.5]$ m are the \textsf{small}, \textsf{medium}, and \textsf{large} value ranges for the step length $d$.}
\end{itemize}
\end{defn}

The properties above imply high-level navigation constraints induced by low-level locomotion dynamics, since the steering ability is constrained by the conditions of $\Delta y_{1} \in \mathcal{R}_{\Delta y_1}$ and $\Delta y_{2} \in \mathcal{R}_{\Delta y_2}$. For example, steering walking with a \textsf{large} $v_{\rm apex}$, may violate Proposition \ref{prop:steering1} or results in $\Delta y_{2,n} \notin \mathcal{R}_{\Delta y_2}$. Similarly, a \textsf{small} step length $d$ and a \textsf{large} $v_{\rm apex}$, may result in $\Delta y_{1,n}$ having the same sign as $\Delta y_{1,c}$, thus accurate tracking of high-level waypoints is not achieved. These properties will be translated into task specifications and embedded in the high-level planner to rule out undesirable actions in the next section.



\section{Task Planning via Belief Abstraction}
\label{sec:task planner}
This section will employ the locomotion keyframe properties above for the high-level task specification design. The goal of our temporal-logic-based task planner is to achieve safe locomotion navigation in a partially observable environment with dynamic obstacles as defined below.
\begin{defn}[Navigation Safety]
Navigation safety is defined as dynamic maneuvering over uneven terrain without falling while avoiding collisions with dynamic obstacles.
\end{defn}



The task planner consists of two components: A high-level navigation planner that plays a navigation and collision avoidance game against the environment on a global coarse discrete abstraction, and an action planner that plays a local navigation game on a fine abstraction of the local environment (i.e., one coarse cell). The action planner generates action sets at each keyframe to achieve the desired coarse-cell transition in the navigation game, which can take multiple walking steps. The reason for splitting the task planner into two layers is that the abstraction granularity required to plan walking actions for each keyframe is too fine to synthesize plans for large environments in a reasonable amount of time.

\subsection{Navigation Planner Design}
The navigation environment is discretized into a coarse two-dimensional grid with a $2.7$ m cell size as shown in Fig. \ref{fig:belief_results}. Each time the robot enters a new cell, the navigation planner evaluates the robot's discrete location ($l_{r,c} \in \mathcal{L}_{r,c}$) and heading ($h_{r,c} \in \mathcal{H}_{r,c}$) on the coarse grid, as well as the dynamic obstacle's location ($l_o \in \mathcal{L}_o$), and determines a desired navigation action ($n_a \in \mathcal{N}_a$). The planner can choose for the robot to stop, or to transition to any reachable safe adjacent cell. $\mathcal{L}_r$ and $\mathcal{L}_o$ denote sets of all coarse cells the robot and dynamic obstacle can occupy, while $\mathcal{H}_{r,c}$ represents the four cardinal directions in which the robot can travel on the coarse abstraction.
Static obstacle locations are encoded into the task specifications. The dynamic obstacle moves under the following assumptions: (a) it will not attempt to collide with the robot when the robot is standing still,
(b) it moves with a fixed speed such that the mobile robot moves to its adjacent coarse cell after each time step, and
(c) it will eventually move out of the way to allow the robot to pass. Assumption (c) prevents a deadlock. 
The task planner guarantees that the walking robot can prevent collisions and achieve a specified navigation goal. 

\subsection{Action Planner Design}
\label{subsec:action_planner}
The local environment, i.e., one coarse cell, is further discretized into a fine abstraction of $26\times 26$ cells. At each walking step, the action planner evaluates the robot's discrete location ($l_{r,f} \in \mathcal{L}_{r,f}$) and heading ($h_{r,f} \in \mathcal{H}_{r,f}$) on the fine grid, as well as the robots current stance foot ($i_{\rm st}$), and determines an appropriate action set $\boldsymbol{a}$.
$\mathcal{H}_{r,f}$ contains a discrete representation of the 16 possible headings the robot could have.
The action planner is responsible for generating a sequence of actions that guarantee that the robot eventually transitions to the next desired coarse cell in the navigation game while ensuring all action sets are safe and achievable based on the current robot and environment states. The key components of the action set are step length ($d \in\{\textsf{small},\textsf{medium},\textsf{large}\}$), heading change ($\Delta\theta \in \Delta\Theta = \{\textsf{left}, \textsf{none}, \textsf{right}\}$), and step height ($\Delta z_{\rm foot} \in \Delta Z_{\rm foot} = \{z_{\rm down2},z_{\rm down1}, z_{\rm flat}, z_{\rm up1},z_{\rm up2}\}$). 
The fine abstraction models the terrain height for each discrete location, allowing the action planner to choose the correct step height $\Delta z_{\rm foot}$ for each keyframe transition. 
The possible heading changes $\Delta\Theta$, correspond to $\{-22.5^\circ, 0^\circ, 22.5^\circ\}$,  are constrained by the minimum number of steps needed to make a $90^\circ$ turn and the maximum allowable heading angle change that results in viable keyframe transitions as defined in Theorem~\ref{thm:viable_trans}. We choose $\Delta\theta = \pm 22.5^\circ$ so that a $90^\circ$ turn can be completed in four steps as can be seen in Fig. \ref{fig:Task_Planner}. Completing the turn in fewer steps is not feasible as it would overly constrain $v_{\rm apex}$, as can be seen in Fig. \ref{fig:vk}(b). 

\subsection{Task Planner Synthesis}
To formally guarantee that the goal locations are reached \textit{infinitely often} while the safety specifications are met, we use General Reactivity of Rank 1 (GR(1)), a fragment of Linear Temporal Logic (LTL). GR(1) allows us to design temporal logic formulas ($\varphi$) with atomic propositions (AP ($\varphi$)) that can either be \textsf{True} ($\varphi \vee \neg\varphi$) or \textsf{False} ($\neg$\textsf{True}). With negation ($\neg$) and disjunction ($\vee$) one can also define the following operators: conjunction ($\wedge$), implication ($\Rightarrow$), and equivalence ($\Leftrightarrow$). There also exist temporal operators ``next" ($\bigcirc$), ``until" ($\mathcal{U}$), ``eventually" ($\diamond$), and ``always" ($\square$). Further details of GR(1) can be found in \cite{gr1}. Our implementation uses the SLUGS reactive synthesis tool \cite{slugs} to design specifications with Atomic Propositions (APs) and natural numbers, which are automatically converted to ones using only APs.


A navigation game structure is proposed by including robot actions in the tuple $\mathcal{G} := (\mathcal{S}, s^{\rm init}, \mathcal{T}_{RO})$ with
\begin{itemize}
    \item $\mathcal{S} = \mathcal{L}_{r,c} \times \mathcal{L}_o \times \mathcal{H}_{r,c} \times \mathcal{A}_n$ is the augmented state;
    \item $s^{\rm init} = (l_{r,c}^{\rm init},l_o^{\rm init},h_{r,c}^{\rm init},n_a^{\rm init}) $ is the initial state;
    \item $\mathcal{T}_{RO} \subseteq \mathcal{S} \times \mathcal{S}$ is a transition relation describing the possible moves of the robot and the obstacle.
\end{itemize}
To synthesize the transition system $\mathcal{T}_{RO}$, we define the rules for the possible successor state locations which will be further expressed in the form of LTL specifications $\psi$. We define the successor location of the robot based on its current state and action $succ_r(l_{r,c},h_{r,c},n_a) = \{l'_{r,c} \in \mathcal{L}_{r,c} | ((l_{r,c},l_o,h_{r,c}),(l'_{r,c},l'_o,h'_{r,c}))\in \mathcal{T}_{RO}\}$. We define the set of possible successor robot actions at the next step as $succ_{n_a}(l_{r,c},l_o,l'_o,h_{r,c}, n_a) = \{ n_a \in \mathcal{N}_a  | ((l_{r,c},l_o,h_{r,c}) , (l'_{r,c},l'_o,h'_{r,c})) \in \mathcal{T}_{RO} \}$. We define the set of successor locations of the obstacle. $succ_o(l_{r,c},l_o,n_a) = \{l'_o \in \mathcal{L}_o | \exists l'_{r,c}, h'_{r,c}. ((l_{r,c},l_o,h_{r,c}) , (l'_{r,c},l'_o,h'_{r,c})) \in \mathcal{T}_{RO} \}$.
Later we will use a belief abstraction inspired from \cite{bharadwaj2018synthesis} to solve our synthesis in a partially observable environment.

The task planner models the robot and environment interplay as a two-player game. The robot action is player 1 while the obstacle is player 2 that is possibly adversarial. The synthesized game guarantees that the robot will always win the game by solving the following reactive problem.

\noindent\textbf{Reactive synthesis problem:}
Given a transition system $\mathcal{T}_{RO}$ and linear temporal logic specifications $\psi$, synthesize a winning strategy for the robot such that only correct decisions are generated in the sense that the executions satisfy $\psi$.

The action planner is synthesized using the same game structure as the navigation planner, with possible states and actions corresponding to Section~\ref{subsec:action_planner}. Since obstacle avoidance is taken care of in the navigation game the obstacle location $\mathcal{L}_o$ and successor function $succ_o$ are not needed for synthesis. Since reactive synthesis is used for both navigation and action planners, the correctness of this hierarchical task planner is guaranteed.


\subsection{Task Planning Specifications}
\label{subsec:nav-spec}
\begin{figure}[t]
\centerline{\includegraphics[width=.4\textwidth]{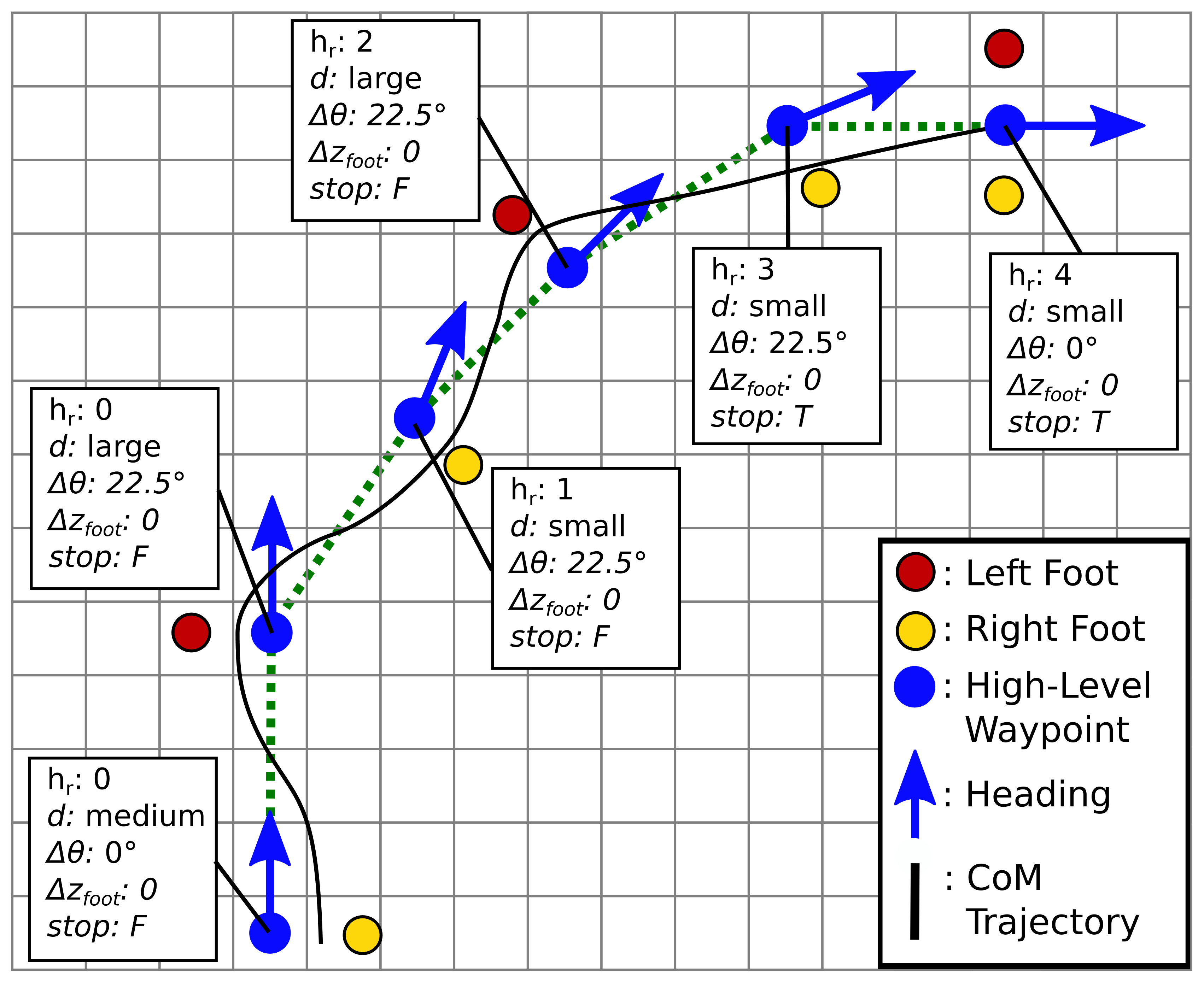}}
\caption{Illustration of fine-level steering walking within one coarse cell.
Discrete actions are planned at each keyframe allowing the robot to traverse the fine grid toward the next coarse cell. A set of locomotion keyframe decisions are also annotated.
}
\label{fig:Task_Planner}
\vspace{-0.1in}
\end{figure}

A set of specifications is needed to describe the possible successor locations and actions ($succ_r$, $succ_{n_a}$, $succ_o$, $succ_{a}$) in the transition system.
To ensure that the $v_{{\rm apex},n}$ safety criteria in Section~\ref{sec:motion planner}-B are met, we introduce a navigation policy that limits $d$ based on $\Delta\theta$, $h_{r,f}$ , and $i_{\rm st}$ in the action planner. The turning strategy ensures that the robot always recovers to the centroid of a cell heading in a cardinal direction ($h_{r,f} = h_{c} \in \{N, E, S, W\}$), this ensures that the same environment transitions happen for a given action and discrete game state. Such a navigation policy is composed through safety specifications governing step length sequences. Here we show an example of these specifications governing the first step of a $90 ^\circ$ turn. Similar specifications exists for other discrete robot states.
\begin{align}\nonumber
        \square \big(( h_{r,f} = h_{c}\wedge ((i_{\rm st} = \textsf{left} \wedge \Delta\theta = \textsf{right}) \\
        \vee (i_{\rm st} = \textsf{right} \wedge \Delta\theta = \textsf{left})) \Rightarrow \bigcirc(d = \textsf{large})\big)
\end{align}
\begin{align}\nonumber
        \square \big(( h_{r,f} = h_{c}\wedge ((i_{\rm st} = \textsf{left} \wedge \Delta\theta = \textsf{left}) \\
        \vee (i_{\rm st} = \textsf{right} \wedge \Delta\theta = \textsf{right})) \Rightarrow \bigcirc(d = \textsf{small})\big)
\end{align}
%
To encode the pickup and drop off task visited infinitely often in the navigation planner, we use two intermediate goal tracking APs $GT_1$ and $GT_2$.
\begin{align}
    (\square \Diamond GT_1) \wedge (\square \Diamond GT_2)
\end{align}
Collision avoidance specifications are also designed but omitted due to space limitations.  

\subsection{Belief Space Planning in Partial Observable Environment}
The navigation planner above synthesizes a reactive, safe game strategy that is always winning in a fully observable environment. However, it is unrealistic to assume that the robot has full knowledge of the environment. We relax this assumption by assigning the robot a visible range within which the robot can accurately identify the obstacle location. To reason about where the out-of-sight obstacle is, we devise an abstract belief set construction method based on the work in \cite{bharadwaj2018synthesis}. Using this belief abstraction, we explicitly track the possible cell locations of the dynamic obstacle, rather than assuming it could be in any non-visible cell.
The abstraction is designed by partitioning regions of the environment into sets of states ($P_e$) and constructing a powerset of these regions ($\mathcal{P}(P_e)$). If the obstacle is in the robot's visible range, its belief state would be a real location in $\mathcal{L}_o$. If it is not in the visible range, then its belief state will be a set of states in $\mathcal{P}(P_e)$ complementing the robot visible region. We define a set of beliefs of obstacle locations as $\mathcal{B}_o = \mathcal{L}_o + \mathcal{P}(P_e) $. Now we define the belief game structure as  $\mathcal{G}_{\rm belief} := (\mathcal{S}_{\rm belief}, s_{\rm belief}^{\rm init}, \mathcal{T}_{\rm belief}, vis)$ with
\begin{itemize}
    \item $\mathcal{S}_{\rm belief} = \mathcal{L}_{r,c} \times \mathcal{B}_o \times \mathcal{H}_{r,c} \times \mathcal{A}_n$;
    \item $s_{\rm belief}^{\rm init} = (l_{r,c}^{\rm init},\{b_o^{\rm init}\},h_{r,c}^{\rm init},n_a^{\rm init})$ is the initial location of the obstacle known a priori;
    \item $\mathcal{T}_{\rm belief} \subseteq  \mathcal{S}_{\rm belief} \times \mathcal{S}_{\rm belief}$ are possible transitions where $((l_{r,c},b_o,h_{r,c},n_a),(l'_{r,c},b'_o,h'_{r,c},n'_a)) \in \mathcal{T}_{\rm belief}$;
    \item $vis : \mathcal{S}_{{\rm belief}} \rightarrow \mathbb{B} $ is a visibility function that maps the state ($l_{r,c}, b_o$) to the boolean $\mathbb{b}$ as \textsf{True} iff $b_o$ is a real location in the visible range of $l_{r,c}$.
\end{itemize}
The successor robot location $l'_{r,c}$ is still defined by $succ_r(l_{r,c},h_{r,c},n_a)$ since the belief of the obstacle location doesn't affect the robot location transitions determined by its current action set. The possible actions at the next step still obey $succ_{n_a}(l_{r,c},l_o,l'_o,h_{r,c},n_a)$ since the dynamic obstacle only affects the possible one-step robot action if it is in the visible range. The set of possible successor beliefs of the obstacle location, $b'_o$, is defined as $succ_{o_b} =\{b'_o \in \mathcal{B}_o | \exists l'_{r,c}, h'_{r,c}. ((l_{r,c},b_o,h_{r,c}) , (l'_{r,c},b'_o,h'_{r,c})) \in \mathcal{T}_{\rm belief} \}$ where $\mathcal{B}'_o \in \mathcal{L}'_o$ when $vis(l_{r,c},l'_o) = \textsf{True}$ and $b'_o \in \mathcal{P}(P_e)$ when $vis(l_{r,c},l'_o) = \textsf{False}$. By correctly specifying the possible successor location of the obstacle based on the current state, the planner is able to reason about how the belief region will evolve and where the obstacle can enter the visible range. 

We did not need to modify $succ_{n_a}$ for the belief game since the action planner remains unchanged. We still incorporate the low-level safety constraints using the same specifications, but allow for a larger set of navigation options than would be possible without tracking the belief of the dynamic obstacle's location. 

\section{Implementation}
\label{sec:results}

We design our coherent planning structure using a combined \textit{top-down} and \textit{bottom-up} strategy. The workflow of our integrated task and motion planner in Fig.~\ref{fig:frame} is: as to the \textit{top-down} strategy, the synthesized navigation and action planners first generate feasible actions based on the navigation and locomotion specifications defined in Section~\ref{subsec:nav-spec}. The keyframe decision-maker then \textit{online} chooses viable keyframe states using the keyframe policy designed in \ref{subsec:keyframe-policy}. Finally, the motion planner generates a locomotion trajectory using the keyframe states. As to the \textit{bottom-up} strategy, properties of the low-level safe keyframe policy are incorporated into the high-level \textit{offline} task planner design. This section evaluates the performance of (i) the high-level task planners by synthesizing a pick-up and drop-off task in a partially observable environment, and (ii) the low-level motion planner by employing our designed keyframe policy to choose proper keyframe states and generating safe locomotion trajectories. The results are simulated using the Drake toolbox \cite{drake} and shown in Fig. \ref{fig:psp}. The code used for implementation is open sourced \href{https://github.com/GTLIDAR/safe-nav-locomotion.git}{\nolinkurl{here}}

\subsection{LTL Task Planning Implementation}

The task planner is evaluated in an environment with multiple static obstacles, one dynamic obstacle, and two rooms with different ground heights separated by a set of stairs. The environment is discretized into a $11\times 5$ coarse grid for navigation planning. 
For action planning, the local environment of each coarse cell is further discretized into a $26\times 26$ fine grid.
Our simulation shows that the robot successfully traverses uneven terrain to complete its navigation goals while steering away from the dynamic obstacle when it appears in the robot's visible range. 

\begin{figure}[t]
\centering
\begin{subfigure}{.5\columnwidth}
  \centering
  \includegraphics[width=.9\linewidth]{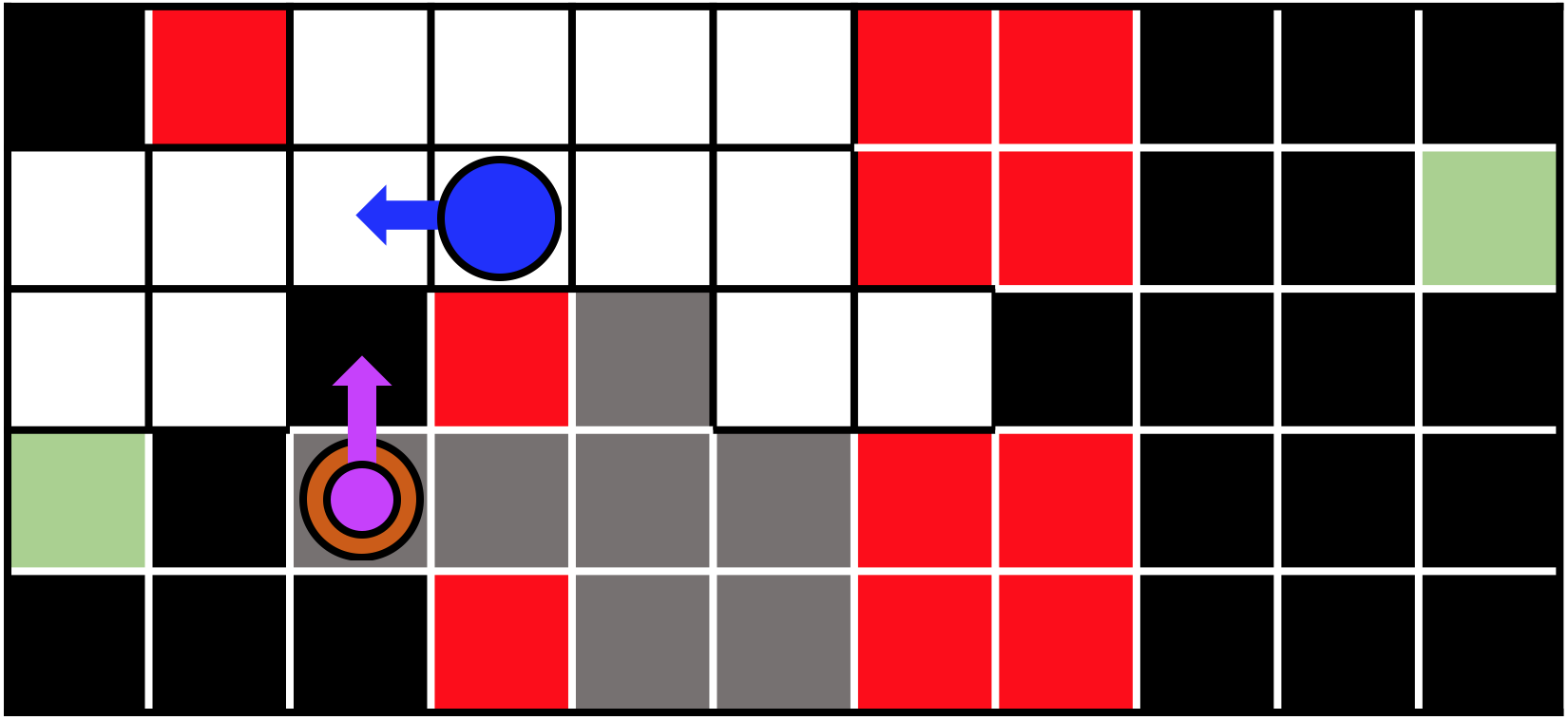}
  \caption{With explicit belief tracking}
  \label{fig:Belief_sub1}
\end{subfigure}%
\begin{subfigure}{.5\columnwidth}
  \centering
  \includegraphics[width=.9\linewidth]{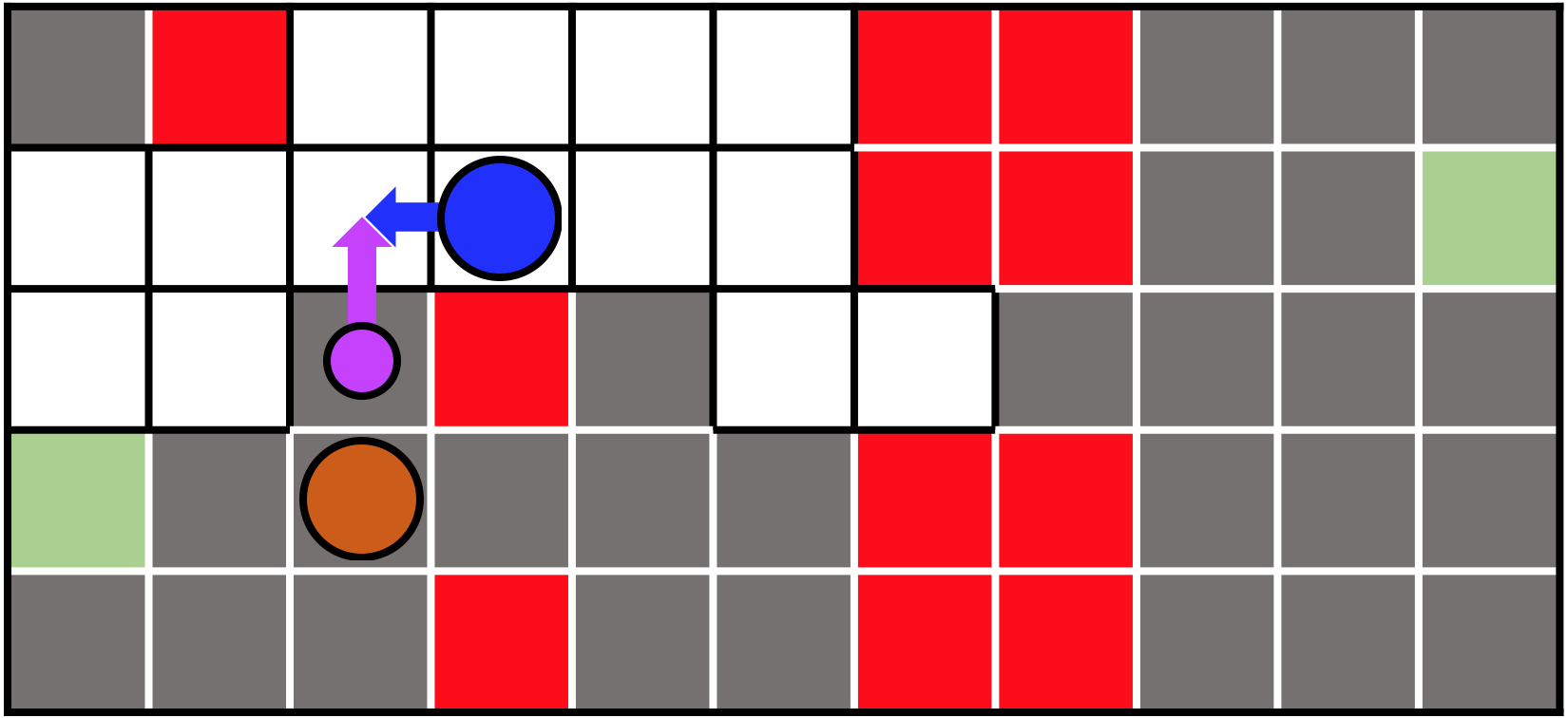}
  \caption{No explicit belief tracking}
  \label{fig:Belief_sub2}
\end{subfigure}
\caption{A snapshot of the coarse-level navigation grid during a simulation where the robot (blue circle) is going between the two goal states (green cells), while avoiding a static obstacle (red cells) and a dynamic obstacle (orange circle). White cells are visible while grey and black cells are non-visible. Gray cells represent the planner's belief of potential obstacle locations. The closest the obstacle could be to the robot, as believed by the planner, is depicted by the pink circle.
}
\label{fig:belief_results}
\vspace{-0.2in}
\end{figure}

The belief structure enables the robot to navigate around obstacles while still guaranteeing that the dynamic obstacle is not in the immediate non-visible vicinity. Fig. (\ref{fig:Belief_sub1}) depicts a snapshot of a simulation where the robot must navigate around such an obstacle to reach its goal states. A successful strategy can be synthesized only when using belief region abstraction.
Without explicitly tracking possible non-visible obstacle locations, the task planner believes the obstacle could be in any non-visible gray cell when it is out of sight. The planner 
is not able to synthesize a strategy that would allow the robot to advance, because it can not guarantee that the obstacle isn't immediately behind the wall.
Fig. (\ref{fig:Belief_sub2}) depicts a potential collision that could occur.
This comparison underlines the significance of the belief abstraction approach.

\begin{figure*}[t]
\centerline{\includegraphics[width=0.9\textwidth]{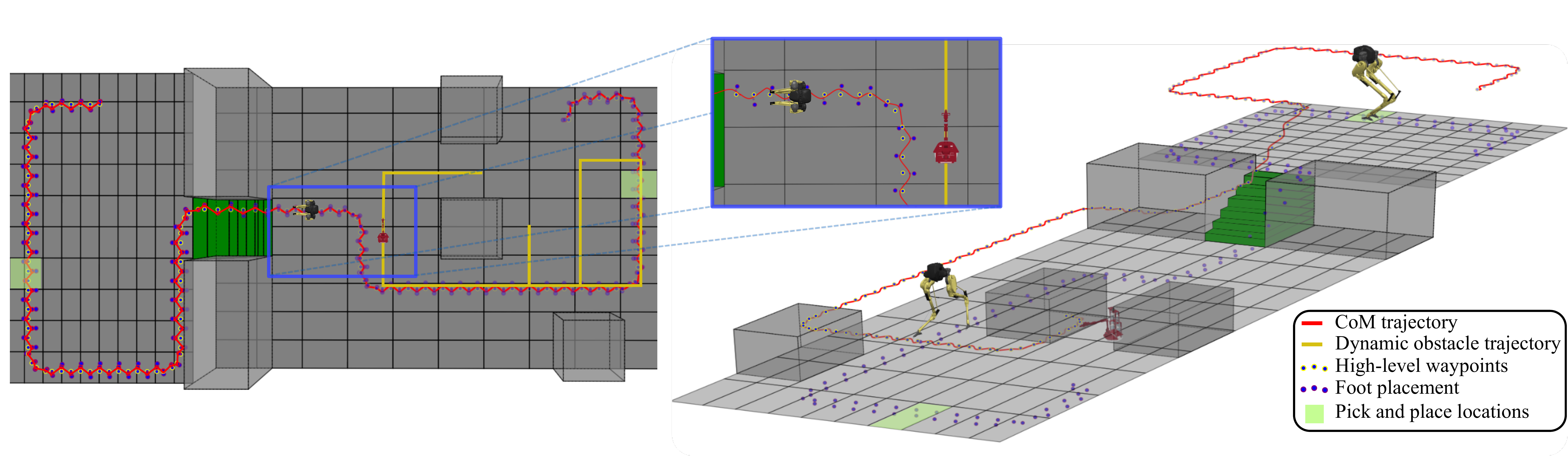}}
\caption{3D simulation of the Cassie robot dynamically navigating in the partially observable environment while avoiding collisions with the mobile robot. Trajectories of Cassie CoM and the moving obstacle as well as Cassie foot placements are illustrated. It has been examined that the navigation trajectory satisfies the proposed locomotion safety and the desired tasks.}
\label{fig:psp}
\vspace{-0.1in}
\end{figure*}

\subsection{Phase-space Locomotion Planner}
The keyframe decision-maker in Section~\ref{subsec:keyframe-policy} determines keyframe states based on the actions from the task planner. It is observed that the proposed keyframe policy guarantees safe locomotion trajectories. The navigation motion is simulated on the Cassie bipedal robot designed by Agility Robotics \cite{agility} in Fig.~\ref{fig:notation}. Cassie's CoM, foot placments and the moving obstacle trajectories are illustrated in Fig.~\ref{fig:psp}. The Cassie walking scenarios include going downstairs, straight walking, stopping, and steering walking. The trajectory satisfies the proposed locomotion safety and the desired navigation task, i.e., the pick and place task at designated locations.


\section{Conclusions}
\label{sec:conclusion}
The proposed task and motion planning framework generated safe locomotion trajectories in a partially observable environment with non-flat terrain. As far as the authors' knowledge, this work takes the first step towards locomotion task and motion planning that incorporates the low-level dynamics into the high-level specification design. This opens up new opportunities for formally designing complex locomotion behaviors reactive to versatile environmental events. 

This framework has the potential to be extended to  more complex environment navigation, such as those with multiple dynamic obstacles or obstacles that move at different speeds.
Our keyframe decision-maker chooses the keyframe policy based on a subset of locomotion dynamics constraints. In the future, we will investigate more rigorous keyframe policy design by exploring extensive locomotion constraints generally.

\section*{Acknowledgment}
The authors would like to express our gratefulness to Suda Bharadwaj and Ufuk Topcu for their discussions on belief abstraction, and Jialin Li for his help in implementing inverse kinematics of the Cassie robot simulation. This work was partially funded by the NSF grant \# IIS-1924978.

\bibliographystyle{IEEEtran}
\bibliography{lidar.bib}

\section*{Appendix}
\label{appendix}
When the CoM motion is constrained within a piece-wise linear surface parameterized by $h = a(x -x_{{\rm foot}})+h_{\rm apex}$, the reduced-order model becomes linear and an analytical solution exists:
\begin{align}
    \label{eqn:appendix1}
    &p(t) = Ae^{\omega t} + Be^{-\omega t} + p_{\rm foot} \\
    \label{eqn:appendix2}
    &\dot{p}(t) = \omega (Ae^{\omega t} - Be^{-\omega t})
\end{align}
where
\begin{align}
    & \omega = \sqrt{\frac{g}{h_{\rm apex}}} \\
    & A = \frac{1}{2}((p_{0}-p_{\rm foot})+\frac{\dot{p}_{0}}{\omega} \\ 
    & B = \frac{1}{2}((p_{0}-p_{\rm foot})-\frac{\dot{p}_{0}}{\omega}
\end{align}
manipulate Eq. (\ref{eqn:appendix1})-(\ref{eqn:appendix2}) gives
\begin{equation}
    p+\frac{\dot{p}}{\omega}-p_{\rm foot}=2Ae^{\omega t}
\end{equation}
which renders
\begin{equation}\label{eqn:appendix3}
    t=\frac{1}{\omega}\log(\frac{p+\frac{\dot{p}}{\omega}-p_{\rm foot}}{2A})
\end{equation}

To find the dynamics, $\dot{p} = f(p)$, which will lead to the switching state solution, remove the $t$ term by plugging Eq. (\ref{eqn:appendix3}) into Eq. (\ref{eqn:appendix1}).
\begin{equation}
    \frac{1}{2}(p-\frac{\dot{p}}{\omega}-p_{\rm foot})= \frac{2AB}{p+\frac{\dot{p}}{\omega}-p_{\rm foot}}
\end{equation}
\begin{equation}
    (p-p_{\rm foot})^{2}-(\frac{\dot{p}}{\omega})^{2}=4AB
\end{equation}
which yields
\begin{equation}
    \dot{p}=\pm \sqrt{\omega^{2}((p-p_{\rm foot})^{2}-(p_{0}-p_{\rm foot})^{2})+\dot{p}_{0}^{2}}
\end{equation}

If the apex height is constant, then $\omega$ is constant. According to the constrain that sagittal velocity should be continuous, the saggital switching position is obtained by 
\begin{equation}
    x_{\rm switch} = \frac{1}{2}(\frac{C}{x_{{\rm foot}, n}-x_{{\rm foot}, c}}+(x_{{\rm foot}, c}+x_{{\rm foot}, n}))
\end{equation}
where
\begin{equation}
\begin{split}
    C = (x_{{\rm apex},c}-x_{{\rm foot},c})^{2}-(x_{{\rm apex},n}-x_{{\rm foot}, n})^{2} \\ 
    + \frac{\dot{x}_{{\rm apex},n}^{2}-\dot{x}_{{\rm apex},c}^{2}}{\omega^{2}}
\end{split}
\end{equation}

\end{document}